\documentclass[
  a4paper,
  USenglish,
  numberwithinsect,
  thm-restate,
  autoref,
  authorcolumns
  ]{lipics-v2021}

\usepackage
{
    graphicx,
    amssymb,
    amsmath,
    amsthm,
    dsfont, %
    xcolor,
    mathtools,
    algorithm,
    algpseudocode,
    comment
}

\usepackage{bm}
\usepackage{csquotes}
\usepackage{enumerate}
\usepackage[inline]{enumitem}
\usepackage[xcolor]{mdframed}
\usepackage{thmtools, thm-restate}

\usepackage{float}

\usepackage{cases}

\usepackage{standalone}									%
\usepackage{tikz}										%
\usetikzlibrary{arrows}
\usetikzlibrary{arrows.meta}
\usetikzlibrary{shapes.geometric}

\usepackage{todonotes}

\usepackage{subcaption}

\usepackage[capitalize, nameinlink, noabbrev]{cleveref}
\usepackage{nicefrac}

\usepackage{datetime}
\usdate

\usepackage{dsfont} %
\usepackage{tikz}
\usepackage{standalone}
\usetikzlibrary{arrows}
\usetikzlibrary{arrows.spaced}
\usetikzlibrary{arrows.meta}
\usetikzlibrary{shapes.geometric}
\usetikzlibrary{shapes.misc}
\usetikzlibrary{backgrounds}

\newcommand{\N}{\mathbb{N}}                                     %
\newcommand{\R}{\mathbb{R}}                                     %
\newcommand{\C}{\mathbb{C}}                                     %
\providecommand{\abs}[1]{\left\lvert #1 \right\rvert}           %

\newcommand{\id}{\mathbf{1}}   %
\newcommand{\eyes}{\mathbf{0}} %

\newcommand{\fsync}{\textsc{$\mathcal{F}$sync}}

\def\oblot/{$\mathcal{OBLOT}$}

\def\LCM/{\textsc{LCM}}
\def\Look/{\textsc{Look}}
\def\Compute/{\textsc{Compute}}
\def\Move/{\textsc{Move}}

\newcommand*{\sym}{\operatorname{sym}}

\newcommand\set[1]{\left\{\,#1\,\right\}}

\makeatletter
\renewcommand*\env@matrix[1][*\c@MaxMatrixCols c]{%
  \hskip -\arraycolsep
  \let\@ifnextchar\new@ifnextchar
  \array{#1}}
\makeatother

\mathtoolsset{centercolon} %

\newlength{\continueindent}
\usepackage{etoolbox}
\makeatletter
\newcommand*{\ALG@customparshape}{\parshape 2 \leftmargin \linewidth \dimexpr\ALG@tlm+\continueindent\relax \dimexpr\linewidth+\leftmargin-\ALG@tlm-\continueindent\relax}
\apptocmd{\ALG@beginblock}{\ALG@customparshape}{}{\errmessage{failed to patch}}
\makeatother
\setuptodonotes{caption={}, size=\tiny}

\newcommand{\trans}{\mathrm{T}}

\renewcommand{\vec}[1]{\mathbf{#1}}
\newcommand{\z}{\vec{z}}

\newcommand{\vzeta}{\boldsymbol{\zeta}}

\newmdenv[
	backgroundcolor=yellow!20,
	leftline=false,
	rightline=false,
	bottomline=false,
	topline=false,
	linewidth=3pt,
	linecolor=blue
	]{myframe}

\newcommand{\GtC}{\textsc{Go-To-The-Center}}

\newcommand{\gta}{\textsc{Go-To-The-Average}}

\newcommand{\gtmlong}{\textsc{Go-to-the-Middle}}
\newcommand{\egtmlong}{$\epsilon$-\textsc{Go-to-the-Middle}}
\newcommand{\egtm}{\ensuremath{\epsilon}\normalfont{\textsc{-GtM}}}
\newcommand{\neargathering}{\textsc{Near-Gathering}}

\newcommand{\wave}[1]{\ensuremath{\normalfont{\texttt{wave}}^{#1}}}
\newcommand{\waveseg}[2]{\ensuremath{\normalfont{\texttt{seg}}^{#1}_{#2}}}
\newcommand{\conboundary}{Connectivity-Boundary}

\newcommand{\wavealgo}{\textsc{Wave-Algorithm}}

\title{Symmetry Preservation in Swarms of Oblivious Robots with Limited Visibility}

\author{Raphael Gerlach}{Institute of Mathematics, Paderborn University, 
Germany}{rgerlach@math.upb.de}{0009-0002-4750-2051}{}
\author{Sören von der Gracht}{Institute of Mathematics, Paderborn University, 
Germany}{soeren.von.der.gracht@uni-paderborn.de}{0000-0002-8054-2058}{Deutsche Forschungsgemeinschaft (DFG, German Research Foundation)--–453112019.}
\author{Christopher Hahn}{University of Hamburg, Department of Informatics, Germany}{christopher.hahn-1@uni-hamburg.de}{0009-0001-7617-6374}{}
\author{Jonas Harbig}{Paderborn University, Heinz Nixdorf Institute,  Germany}{jonas.harbig@uni-paderborn.de}{0000-0003-3943-5979}{Deutsche Forschungsgemeinschaft (DFG, German Research Foundation)--–453112019.}
\author{Peter Kling}{University of Hamburg, Department of Informatics, Germany}{peter.kling@uni-hamburg.de}{0000-0003-0000-8689}{}
\funding{This work was partially supported by the German Research Foundation(DFG) under the project numbers ME 872/14-1 \& 453112019.}

\authorrunning{R. Gerlach, S. v. d. Gracht, C. Hahn, J. Harbig and P. Kling,}

\Copyright{Raphael Gerlach and Sören von der Gracht and Christopher Hahn and Jonas Harbig and Peter Kling}



\nolinenumbers
\hideLIPIcs

\EventEditors{John Q. Open and Joan R. Access}
\EventNoEds{2}
\EventLongTitle{42nd Conference on Very Important Topics (CVIT 2016)}
\EventShortTitle{CVIT 2016}
\EventAcronym{CVIT}
\EventYear{2016}
\EventDate{December 24--27, 2016}
\EventLocation{Little Whinging, United Kingdom}
\EventLogo{}
\SeriesVolume{42}
\ArticleNo{23}

\begin{document}
\maketitle

\begin{abstract}
In the \emph{general pattern formation} (GPF) problem, a swarm of simple autonomous, disoriented robots must form a given pattern.
The robots' simplicity imply a strong limitation:
When the initial configuration is rotationally symmetric, only patterns with a similar symmetry can be formed \cite{DBLP:journals/tcs/YamashitaS10}. %
The only known algorithm to form large patterns with limited visibility and without memory requires the robots to start in a \emph{near-gathering} (a swarm of constant diameter) \cite{DBLP:conf/sand/HahnHK24}. %
However, not only do we not know any near-gathering algorithm guaranteed to preserve symmetry but most natural gathering strategies trivially increase symmetries \cite{DBLP:conf/opodis/CastenowH0KKH22}.  %

Thus, we study near-gathering without changing the swarm's rotational symmetry for disoriented, oblivious robots with limited visibility (the \oblot/-model, see \cite{DBLP:series/lncs/FlocchiniPS19}). %
We introduce a technique based on the theory of dynamical systems to analyze how a given algorithm affects symmetry and provide sufficient conditions for symmetry preservation.
Until now, it was unknown whether the considered \oblot/-model allows for \emph{any} non-trivial algorithm that always preserves symmetry.
Our first result shows that a variant of \gta{} \emph{always} preserves symmetry but may sometimes lead to multiple, unconnected near-gathering clusters.
Our second result is a symmetry-preserving near-gathering algorithm that works on swarms with a convex boundary (the outer boundary of the unit disc graph) and without \enquote{holes} (circles of diameter $1$ inside the boundary without any robots).

 \end{abstract}

\section{Introduction}%
\label{sec:introduction}

The study of large \emph{robot swarms} that consist of simple (and thus cheap) mobile, autonomous robots has grown into an important and active research area in recent years.
For example, when used for exploration or rescue missions in hazardous environments (like the deep sea or outer space), such swarms can be more robust and economical compared to a small group of more capable but expensive high-end units~\cite{kang2019marsbee}.
At a completely different scale, precision medicine explores how to use swarms of nanobots (with inherently limited capabilities due to their small size) to, e.g., deliver drugs in a more targeted manner~\cite{DOI:10.1002/advs.202002203}.

A key question in this area is to what extent one can make up for the robots' lack of individual capabilities by clever coordination, possibly even reaching emergent behavior of the swarm as a whole.
Examples for the high potential of such an approach can be readily found in nature, where ant colonies or bee hives show a remarkable degree of complexity built on rather simple interaction rules between a vast number of very simple entities~\cite{DOI:10.1126/science.1210361, DBLP:conf/ecal/KhuongTJPG11}.
Similar to many other branches of distributed computing, breaking or avoiding \emph{symmetries} presents a major challenge for such systems.

\subparagraph{General Pattern Formation in the \oblot/ Model.}
Our work studies the role of symmetries in solving one of the most important problems in the theory of swarm robots, the \emph{general pattern formation} (GPF) problem.
Here, a swarm of $n \in \N$ autonomous, mobile robots must form a pattern $P \subseteq \R^2$ of $\abs{P} = n$ coordinates (in an arbitrary rotation/translation).
We assume the well-known \oblot/ model~\cite{DBLP:series/lncs/FlocchiniPS19} for deterministic point robots that are \emph{oblivious} (have no memory), \emph{anonymous} (have no identities and cannot be distinguished), \emph{homogeneous} (execute the same algorithm), and \emph{disoriented} (perceive their surroundings in coordinate systems that may be arbitrarily rotated/translated compared to other robots).
We also assume that the robots have a limited, constant \emph{viewing range}, such that they cannot observe anything beyond that range.
Robots act in discrete time steps in which each of them performs a full LCM-cycle consisting of a \Look/- (observe surroundings), a \Compute/- (calculate target), and a \Move/- (move to target) phase.
This is often referred to as the \emph{fully-synchronous} (\fsync) time model (other models allow the phases of different robots to be arbitrarily interleaved).
We consider only \emph{rigid} movements (robots always reach their designated target).

\subparagraph{Symmetries in GPF.}
Symmetries play a key role in GPF, since – even if the robots are not oblivious and have an unlimited viewing range – a swarm can only form patterns whose \emph{symmetricity} (a measure of a pattern's symmetry, see \cref{def:symmetricity}) is a multiple of the swarm's initial symmetricity~\cite{DBLP:journals/siamcomp/SuzukiY99, DBLP:journals/siamcomp/FujinagaYOKY15}.
In the case of oblivious robots with an unlimited viewing range, even in an asynchronous setting, this \emph{symmetricity condition} characterizes \emph{exactly} those patterns that can be formed~\cite{DBLP:journals/tcs/YamashitaS10, DBLP:journals/siamcomp/FujinagaYOKY15}.
A good overview of the symmetricity condition's role in various synchronous and asynchronous settings is given in~\cite{DBLP:series/lncs/Yamauchi19}.

If the viewing range is limited but robots are not oblivious, \cite{DBLP:conf/sirocco/YamauchiY13} gave a protocol that forms a scaled-down version of a target pattern $P$ adhering to the symmetricity condition.
That protocol first forms a \emph{\neargathering{}} (a formation whose diameter is at most the robots' viewing range) and then uses the de facto global view to form a small version of $P$.
The authors also show that if movements are \emph{non-rigid} (i.e., an adversary can stop robots during their movement), robots must be non-oblivious to solve GPF.

Under the light of these results, it remains an open question whether oblivious robots with a limited viewing range can form arbitrary connected patterns $P$ (ideally in its original size) under the above symmetricity condition.
Very recently, \cite{DBLP:conf/sand/HahnHK24} introduced a protocol that achieves this \emph{if the robots start from a \neargathering}.
This opens the possibility to solve GPF from \emph{any} formation by first (similiar to~\cite{DBLP:conf/sirocco/YamauchiY13}) forming a \neargathering{} and then (once the robots observe $\abs{P}$ peers in the \neargathering) form $P$ using the protocol from~\cite{DBLP:conf/sand/HahnHK24}.

Unfortunately, while there are algorithms for local, oblivious robots that achieve a \neargathering{}~\cite{DBLP:conf/opodis/CastenowH0KKH22, DBLP:journals/dc/PagliPV15}, all of them might (and typically do) \emph{increase} the initial symmetry.
As a result, it might be impossible to form the target pattern from the resulting \neargathering{}, even if the swarm's original symmetry met the symmetricity condition.
The authors of~\cite{DBLP:conf/sand/HahnHK24} leave as a central open question \emph{\enquote{whether there is a suitable \neargathering{} protocol that preserves the initial symmetricity}} for oblivious robots\footnote{Note that~\cite{DBLP:conf/sirocco/YamauchiY13} uses the robots' memory to remember their movements, which allows them to maintain their original position and, thus, symmetricity (symmetric robots in the \neargathering{} can be distinguished by their original position).}.
If the answer were positive, one could show that also for oblivious robots with a limited viewing range in the synchronous model and with rigid movements, the patterns that can be formed are characterized by the symmetricity condition.

\begin{figure}
    \begin{subfigure}{0.48\linewidth}
        \includegraphics[page=3,width=\linewidth]{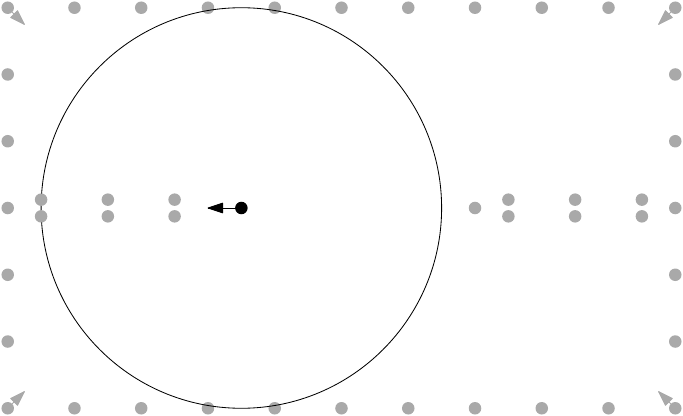}
        \caption{Example with hole. \Cref{alg:near-gath-convex-no-holes}: \neargathering\ but introduces new symmetries\footnotemark[1]. \gta: \neargathering\ with symmetry preservation.\footnotemark[2]
        }
        \label{fig:gta-better-egtm}
    \end{subfigure}
    \hfill
    \begin{subfigure}{0.48\linewidth}
        \hfill
        \includegraphics[page=2,width=0.8\linewidth]{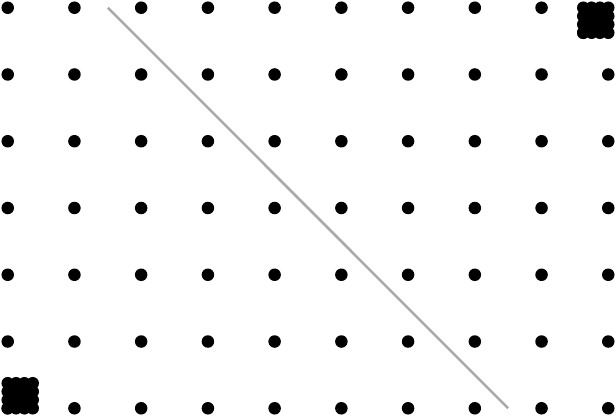}
        \caption{Example with convex boundary and without hole. \Cref{alg:near-gath-convex-no-holes}: \neargathering\ with symmetry preservation. \gta: Keeps symmetry but looses connectivity (swarm splits at the gray line)\footnotemark[3].}
        \label{fig:egtm-better-gta}
    \end{subfigure}
    \caption{
        }
    \label{fig:examples-egtm-vs-gta}
\end{figure}
{
    \renewcommand{\thefootnote}{\alph{footnote}}
\footnotetext[1]{{The black robot does not observe robots on the outer rectangle (circle depicts viewing range). It assumes it is a boundary-robot and moves accordingly. All other robots in the inside of the rectangle do not move, because they see enough robots on the outer rectangle to detect, that they are inner-robots.}}
\footnotetext[2]{We simulated this example with the shown viewing range.}
\footnotetext[3]{The distance between grid points is $\nicefrac{1}{\sqrt{2}}$ and the viewing range $2 + \sqrt{2}$. We simulated it with 400 robots at each cluster.}
}

\subparagraph{Our Contribution.}
In this work, we initiate the systematical study of when and how a swarm of oblivious robots with limited viewing range can perform global tasks like \neargathering{} without increasing the swarm's initial symmetricity.
We derive a mathematical framework based on methods from the theory of dynamical systems.
In particular, we formulate the following simple but useful \lcnamecref{thr:main-symmetries-math} (see \cref{sec:preliminaries}) that provides sufficient properties for a given swarm protocol (represented by its evolution function $F\colon \R^{2n} \to \R^{2n}$ that describes the configuration $\z^+ \coloneqq F(\z)$ after one protocol step on a given configuration $\z \in \R^{2n}$ of $n$ robots in the Euclidean plane) to preserve symmetricity.
For the precise definitions of the used terms, like a configuration $\z$'s symmetries, we refer to \cref{sec:preliminaries}.
\begin{theorem}[{name=, restate=[name=restated]thmSymPreservSuff}]%
\label{thr:main-symmetries-math}%
Consider the dynamics of an arbitrary swarm protocol with evolution function $F\colon \R^{2n} \to \R^{2n}$.
Assume that $F$ is (locally) invertible.
Then, any configuration $\z \in \R^{2n}$ and its successor configuration $\z^+ \coloneqq F(\z)$ have the same symmetries $G_{\z^+} = G_{\z}$.
\end{theorem}

The framework of dynamical systems provides a clean mathematical basis to formulate the symmetries of a given configuration and how they are affected by a protocol step.
To prove the usefulness of our framework, we provide two example protocols that, under certain conditions, achieve a \neargathering\ without increasing the swarms symmetricity.

The first protocol (see \cref{sec:protocol:averaging}) is a variant of the well known \gta{} protocol.
This protocol is known to \emph{not} always preserve the swarm's initial connectivity (see \Cref{fig:egtm-better-gta} for an example), but if it does, it leads to \neargathering\footnote{Robots on the swarm's convex hull move inwards in every step.}.
Our framework easily implies that it preserves the swarm's initial symmetry.

Our second protocol (see \cref{sec:protocol:contractingwaves}) is an adaption of the well known \gtmlong{} strategy.
It restricts movements to robots close to the swarm's boundary and coordinates the movement of nearby robots to ensure that no symmetries are created.
While this protocol always preserves connectivity and leads to \neargathering\ (\Cref{remark:egtm-always-near-gathering}),
it is not always symmetry preserving (see \Cref{fig:gta-better-egtm} for an example). 
We can prove that its evolution function is locally invertible (and, thus, the protocol symmetry preserving) for configurations that contain no \enquote{holes} and are convex (we formalize this requirement in \cref{sec:protocol:contractingwaves}).

\subparagraph{Outline.}
This work is outlined as follows: First, further related work is reviewed in \Cref{sec:related_work}.
In \cref{sec:preliminaries}, we give a formal description of the problem and formulate a sufficient condition for the preservation of symmetries based on the theory of dynamical systems.
We present and analyze our two algorithms in \cref{sec:protocol:averaging} and \cref{sec:protocol:contractingwaves}.
Finally, we discuss limitations and possible generalizations of both algorithms in \cref{sec:conclusion}.
\section{Further Related Work}%
\label{sec:related_work}

The general pattern formation problem has been considered in numerous further settings and variants whose focus or assumptions differ from our setting and the work described in the \cref{sec:introduction}.
For example, \cite{DBLP:journals/tcs/BoseAKS20} considers pattern formation for oblivious robots with unlimited visibility on an infinite grid, showing that an initially asymmetric swarm can form any pattern (basically thanks to the grid enabling partial axis agreement).
The authors of~\cite{DBLP:journals/tcs/BoseKAS21} assume unlimited but obstructed visibility (i.e., robots can obstruct each others view) with partial axis agreement and luminous robots (that can communicate via a constant number of lights).
The role of partial axis agreement and disorientation was also studied in~\cite{DBLP:journals/tcs/FlocchiniPSW08, DBLP:journals/dc/CiceroneSN19}, proving several possibility and impossibility results depending on how much of their coordinate systems the robots agree.
In~\cite{DBLP:journals/dc/0001FSY15}, the authors considered when oblivious robots with unlimited visibility can form (cyclic) sequences of patterns.
Note that the impossibility results in most of these works typically rely to a large degree on assuming an asynchronous model (where the robots' LCM-cycles may be arbitrarily interleaved) and/or non-rigid movements (where an adversary can stop robots mid-motion).
For further results on general pattern formation we refer to the survey~\cite{DOI:10.1088/1742-6596/473/1/012016}.

There are also several results on forming specific patterns, like
\begin{itemize*}[afterlabel=, label=]
\item a point~\cite{DBLP:conf/spaa/DegenerKLHPW11, DBLP:conf/opodis/CastenowH0KKH22, DBLP:journals/tcs/CastenowFHJH20} (gathering),
\item an arbitrarily tight near-gathering~\cite{DBLP:journals/siamcomp/CohenP05, DBLP:conf/podc/KirkpatrickKNPS21} (convergence), and
\item a uniform circle~\cite{DBLP:series/lncs/Viglietta19, DBLP:conf/icdcit/MondalC20, DBLP:journals/dc/FlocchiniPSV17}.
\end{itemize*}
A rather up-to-date and good overview can be found in~\cite{DBLP:series/lncs/11340}.

In dynamical systems theory, much of the related literature considers \emph{consensus} or \emph{synchronization} problems that share characteristics with gathering or near-gathering, e.g., by identifying robot positions with \enquote{opinions} the agents want to find a consensus on. A good overview over this research branch can be found in~\cite{Pikovskij.2003} or in the slightly more recent surveys~\cite{Chen.2013, Dorfler.2014}. Extensions of these methods to general pattern formation (in our sense) have been proposed as well (see e.g., \cite{Almuzakki.2023}). However, in this area the focus is typically on time-continuous systems---in the form of differential equations---and the models are not as strictly restricted as necessary in our context, e.g., allowing global communication range.

\section{Preliminaries \& Notation}%
\label{sec:preliminaries}

This \lcnamecref{sec:preliminaries} introduces some formal notation and definitions we use throughout the rest of the paper.
In particular, we use some tools from the theory of dynamical systems to formulate sufficient conditions for swarm protocols that preserve symmetries.

\subparagraph{Basic Notions \& Notation.}

When discussing the configuration (state) of the swarm in some round $t \in \N$, we typically take the perspective of an external observer who passively observes the motion of the robot swarm.
In particular, given a swarm of $n$ robots we specify the robots' positions $z_1^t=(x_1^t,y_1^t),\dotsc,z_n^t=(x_n^t,y_n^t)$ in round $t$ in an arbitrary \emph{global coordinate system} for $\R^2$.
The \emph{configuration} of the entire swarm in round $t$ can then be expressed as the (column) vector $\z^t={(z_i^t)}_{i=1}^{n}\in(\R^2)^n\equiv\R^{2n}$.
During the analysis, we sometimes identify a robot with its position (e.g., saying that robot $z_i$ moves to its target point).

The \emph{unit disc graph} of a configuration $\z$ is an undirected graph $G = (V, E)$ with $V = \set{1, \dots, n}$ and $\set{i, j} \in E$ if and only if $\|z_i - z_j\| \leq 1$.
A \emph{\neargathering} is a configuration whose unit disc graph is a clique (i.e., any two robots see each other).
We say a configuration is \emph{connected} if its unit disc graph is connected.

The protocol definition given in \cref{sec:protocol:contractingwaves} relies on the \emph{\conboundary} of the swarm's configuration, which we define as the outer boundary of the (in general non-convex) polygon enclosed by the unit disc graph.
In our analysis we sometimes identify the \conboundary{} with the set of robots that lie on it.

\subsection{Characterizing Protocols via their Evolution Function}%
\label{sec:evolutionfunction}

We now formalize how to model the effect of a given protocol in the language of dynamical systems, which we will use to analyze the protocol's effect on the swarm's symmetry.

Remember that robots move autonomously in the plane in discrete synchronous rounds according to the same protocol.
Mathematically, from round $t$ to $t + 1$, the $i$-th robot moves from its old position $z_i^t$ to its new position $z_i^{t+1}$.
This new position $z_i^{t+1}$ depends on its own as well as on (potentially) \emph{all other robots'} positions (e.g., if the robot currently sees the whole swarm and the protocol states to move towards the center of gravity of visible robots).
In the most general formulation, we can describe such a dynamics as
\begin{equation}
	\label{eq:general_robot}
	z_i^{t+1} = f(z_i^t; \z^t) \quad i=1,\dotsc, n
\end{equation}
for some function $f\colon \R^2 \times \R^{2n} \to \R^2$.
Note that all robots execute the same protocol, so the function $f$ does not depend on $i$.
However, since all robots have their own coordinate system they must to be able to distinguish their own position from all other positions in the swarm.
This is reflected in the explicit first argument given to $f$.

With this, we can describe the evolution of the entire configuration as
\begin{equation}
	\label{eq:general_configuration}
	\z^{t+1} = F(\z^t) = \begin{pmatrix} f(z_1^t; \z^t) \\ \vdots \\ f(z_n^t; \z^t)
	\end{pmatrix}
    .
\end{equation}
We refer to $F$ as the \emph{evolution function} of the protocol.

Whenever the specific value of $t$ is irrelevant (e.g., when we investigate an arbitrary round) we drop the superscript and abbreviate $z_i^+ = f(z_i; \z)$ where $z_i^+\in\R^2$ indicates the \enquote{next} position of robot $i$.
Similarly, we use the notation $\z^+ = F(\z)$ to indicate the \enquote{next} configuration if the evolution function $F$ is applied to some configuration $\z$.

\subsection{Preserving Symmetries via Local Invertibility}%
\label{sec:preservingsymmetries}

As explained in the introduction, wee seek \neargathering{} strategies that do not increase the swarm's \emph{symmetricity} (because of the \emph{symmetricity condition} for pattern formation, see~\cite[Theorem~1]{DBLP:journals/siamcomp/FujinagaYOKY15}).
The symmetricity measures the rotational symmetricity of a finite set $P \subseteq \R^2$ (in our case the set of robot positions) and can be defined as follows:
\begin{definition}[Symmetricity~\cite{DBLP:journals/siamcomp/FujinagaYOKY15}]%
\label{def:symmetricity}
Consider a finite set $P \subseteq \R^2$ whose smallest enclosing circle is centered at $c \in \R^2$.
A \emph{$m$-regular} partition of $P$ is a partition of $P$ into $k = \abs{P}/m$ regular $m$-gons with common center $c$.
The \emph{symmetricity} of $P$ is defined as
\begin{math}
\sym(P)
\coloneqq
\max\set{ m \in \N \mid \text{there is a $m$-regular partition of $P$} }
\end{math}.
\end{definition}
In \cref{def:symmetricity}, a single point is considered a $1$-gon with an arbitrary center.
Thus, any $P$ has a $1$-regular partition.
Note that, if the center $c$ is an element of $P$, then $\sym(P) = 1$.\footnote{
    One might assume a $n$-gon together with its center forms a rather symmetric set of size $n+1$.
    But robots can easily break the perceived symmetry, since the center robot basically functions as a leader.
}
Since here we consider the configuration typically in the global coordinate system of the external observer (which can be chosen arbitrarily), we assume (without loss of generality) that the swarm's center $c$ is the origin.

We now formalize the notion of a \emph{symmetry} in such a way that we can apply the theory of dynamic systems to argue how a protocol's evolution function influences those symmetries.
For a swarm of symmetricity $m > 1$ there are exactly $m$ rotations $\rho\colon \R^2 \to \R^2$ around the origin (center) under which the \emph{set} of robot positions is invariant (i.e., $\set{\rho(z_1), \dotsc, \rho(z_n)} = \set{z_1, \dotsc, z_n}$).
We represent configurations as tuples instead of sets, which is more typical in the context of dynamical systems.
Thus, we define a symmetry of a configuration $\z \in (\R^2)^n$ as follows (using a permutation to \enquote{relabel} the tuple suitably after the rotation).
\begin{definition}[Symmetry of a Configuration 1]%
\label{def:symmetry1}
Consider a configuration $\z=(z_1,\dotsc,z_n)^\trans$ and a rotation $\rho\colon\R^2\to\R^2$ centered at the origin. Then $\rho$ is a \emph{symmetry of the configuration $\z$} if and only if there exists a permutation $\kappa\colon\set{1,\dotsc,n}\to\set{1,\dotsc,n}$ such that
\begin{math}
\rho(z_i) = z_{\kappa(i)}
\end{math}
for all $i \in \set{1, \dots, n}$.
\end{definition}

To lift the rotation $\rho$ and permutation $\kappa$ from \cref{def:symmetry1} to the entire configuration $\z \in (\R^2)^n \equiv \R^{2n}$, we use the following block matrices $M_{\rho}$ and $M_{\kappa}$ ($n \times n$ matrices with entries in $\R^{2 \times 2}$) implied by $\rho$ and $\kappa$:
\begin{equation}%
\label{eq:symm-mat}
\begin{split}
M_\rho &=
\begin{pmatrix}
    \rho & \eyes & \cdots & \eyes \\
    \eyes & \rho & \ddots & \vdots \\
    \vdots & \ddots & \ddots & \eyes \\
    \eyes & \cdots & \eyes & \rho
\end{pmatrix}
\qquad\text{and}\qquad
(M_\kappa)_{ij} &= \begin{cases}
    \id_2, \quad &\text{if } \kappa(i)=j\\
    \eyes, \quad &\text{otherwise}.
\end{cases}
\end{split}
\end{equation}
Here, $\id_l$ denotes for the $l \times l$ identity matrix and $\eyes$ the zero matrix of suitable dimensions.
Without further mention, we identify both matrices with their $\R^{2n\times2n}$ counterparts.

The matrices above allows us to reformulate the condition from \cref{def:symmetry1} as
\begin{math}
M_\rho \z = M_\kappa \z
\end{math}.
Since $M_{\kappa}$ is furthermore invertible with $M_{\kappa}^{-1} = M_{\kappa^{-1}}$ (and the inverse $\kappa^{-1}$ is also a permutation), we can reformulate \cref{def:symmetry1}:
\begin{definition}[Symmetry of a Configuration (alternative)]
Consider a configuration $\z \in \R^{2n}$ and a rotation $\rho\colon\R^2\to\R^2$, both centered at the origin.
Then $\rho$ is a \emph{symmetry of the configuration $\z$} if and only if there exists a permutation $\kappa\colon\set{1,\dotsc,n}\to\set{1,\dotsc,n}$ such that
\begin{math}
M_\kappa M_\rho\z = \z
\end{math}.
\end{definition}

We denote the set of all \emph{potential symmetries} as
\begin{equation}%
\label{eq:G}
G = \set{M_\kappa M_\rho \mid \kappa\colon\set{1,\dotsc,n}\to\set{1,\dotsc,n} \text{permutation}, \rho\colon\R^2\to\R^2 \text{ rotation}}.
\end{equation}
and the subset of \emph{(actual) symmetries} of a configuration $\z$ as
\begin{equation}%
\label{eq:Gz}
G_\z = \set{M_\kappa M_\rho\in G \mid M_\kappa M_\rho\z = \z}.
\end{equation}

With this formalization at hand, we can study how a protocol (via its evolution function) influences a configurations actual symmetries.
In fact, the classical theory of equivariant dynamics~\cite{Golubitsky.1988,Chossat.2000} immediately yields that a protocol can never cause the \emph{loss} of symmetries\footnote{
    For example, robots forming an identical $n$-gon have identical views and, thus, perform the same local calculations.
    Thus, the swarm would be forever trapped in a, possibly scaled, $n$-gon formation.
}
Our \cref{thr:main-symmetries-math} states that if the evolution function $F$ is additionally invertible, then we also cannot \emph{gain} symmetries.
Its proof is given in \cref{sec:proofthm1}.
\thmSymPreservSuff*

Note that the assumption of invertibility is required on the level of the configuration.
In particular, a robot \emph{does not} have to be able to determine its previous position based on its local information.
Informally speaking, from our perspective as an external observer we must always be able to determine the swarm's configuration in the previous round.\footnote{
    The statement is in fact stronger, since we require only \emph{local} invertibility (i.e., for any configuration $\z$ there are open subsets $U_{\z}$ and $V_{F(\z)}$ containing $\z$ and $F(\z)$, respectively, such that $F_{\z}\colon U_{\z} \to V_{F(\z)})$ is invertible).
    This is, for example, central for the analysis of the averaging strategy from \cref{sec:protocol:averaging}.
}

Moreover, it is important to emphasize that the \lcnamecref{thr:main-symmetries-math} does not require any regularity of the evolution function $F$.
In particular, it remains true even if $F$ is non-continuous, which will be essential for the protocol presented in \Cref{sec:protocol:contractingwaves} below.
\section{Protocol 1: Preserving Symmetries via Averaging}%
\label{sec:protocol:averaging}

This \lcnamecref{sec:protocol:averaging} presents a protocol that preserves
the symmetries of a configuration
(which we prove via \cref{thr:main-symmetries-math}).
However, it does not always achieve \neargathering{} as certain initial configurations may result in several clusters of \neargathering s.
Our proposed protocol $\varepsilon$-\gta{} does the following in each round:
\begin{enumerate}
\item Observes positions of visible neighbors (viewing range $1$).
\item Calculate the \emph{target point} as the \emph{weighted} (see below) average of all visible robot positions.
\item Move an $\varepsilon$-fraction towards the target point for an $\varepsilon\in(0,1)$.
\end{enumerate}

The weights of the $i$-robot at position $\z_i$ for a visible neighbor at position $\z_j$ is derived via a monotonically decreasing \emph{bump function} of their squared distances $X \coloneqq \|z_i - z_j\|^2$ defined via
\begin{equation}%
\label{eq:ex_bump0}
b(X) = \begin{cases}
    \exp \left(-\frac{X^2}{1-X^2}\right) & \text{if }X \in [0,1]\\
    0 & \text{if } X > 1.
\end{cases}
\end{equation}
whose graph is shown in \cref{fig:ex_bump0}.
\begin{figure}
\includegraphics[width=0.9\linewidth]{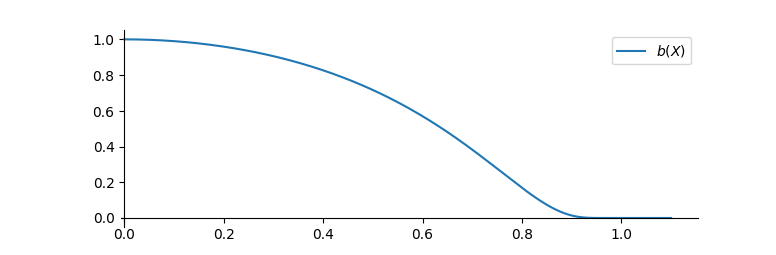}
\caption{%
    Graph of bump function $b(X)$.
    It decreases monotonically from $1$ (for $X = \|z_i - z_j\|^2 = 0$; i.e., when $j$ is at the same position as $i$) to $0$ (for $X = 1$; i.e., when $j$ is at the brink of being invisible).
}\label{fig:ex_bump0}
\end{figure}
With the bump function, we can model the computation of the target point of a robot in position $z_i\in\R^2$ of a configuration $\z\in\R^{2n}$ as
\begin{equation}%
\label{eq:ex_target}
T(z_i;\z) = \frac{1}{n} \sum_{i=j}^n b(\|z_i-z_j\|^2)(z_j-z_i).
\end{equation}
Note that the weights (values of the bump function) for robots outside of the viewing range of $i$ are $0$ and, hence, do not affect the computation.
With this, we can describe the evolution of position $z_i$ in the form of \cref{eq:general_robot} by a (robot-independent) function $f\coloneq \R^2 \times \R^{2n} \to \R^2$ defined by
\begin{equation}%
\label{eq:ex_dynamics}
z_i^+
=
f(z_i; \z)
=
z_i + \varepsilon \cdot T(z_i;\z)
=
z_i + \frac{\varepsilon}{n} \sum_{i=j}^n b(\|z_i-z_j\|^2) \cdot (z_j-z_i)
\end{equation}
for some fixed $\varepsilon \in (0,1)$.
This yields our protocol's evolution function $F$ as specified in \cref{eq:general_configuration}.

The next \lcnamecref{lem:prot1:invertibleF} states that if $\epsilon$ is chosen small enough, $F$ is locally invertible.
As a direct consequence of \cref{thr:main-symmetries-math}, this implies that $\varepsilon$-\gta{} preserves symmetries.
The proof of the following \lcnamecref{lem:prot1:invertibleF} can be found in \cref{ap:protocol1}.
\begin{lemma}%
\label{lem:prot1:invertibleF}
Consider the evolution function $F$ of the $\varepsilon$-\gta{} protocol for $\varepsilon < \frac{n}{27(n-1)}$.
Then $F$ is locally invertible.
\end{lemma}

Unfortunately, as with the standard \gta{} protocol, we cannot guarantee that connectivity is preserved by this adaptation.
Thus, in general we might not end up in a \neargathering, but in several clusters at distance greater than $1$, each of which can be seen as a \enquote{separate} \neargathering s.
However, if we start in a configuration for which \gta{} maintains connectivity (like highly regular meshes), we achieve \neargathering{} without increasing the symmetries.
It remains an open question to characterize such configurations more clearly (cf.~\cref{sec:conclusion}).

\section{Protocol 2: Preserving Symmetries via Contracting Waves}%
\label{sec:protocol:contractingwaves}
In this section, we will provide a symmetry preserving algorithm that yields \neargathering.
However, it only works on a subset of initial configurations and it needs a larger viewing range.

\subparagraph{Requirements of the algorithm.}
Remember the \conboundary\ of a configuration defined in \cref{sec:preliminaries}.
A \emph{$\delta$-hole} of a configuration is a circular area inside the \conboundary\ with a diameter of $\delta$ that contains no robot.

For our protocol we require that the swarm starts in a configuration that contains no $1$-hole and that has a convex\footnote{Note, that we do not consider it to be strictly convex. It may contain multiple collinear robots.} \conboundary.
The robots have a viewing range of $2 + \sqrt{2}$.

Usually, algorithms with limited visibility allow initial configurations with a connected unit disc graph.
Our requirements only allow a subset of these configurations.
However, the subset still allows a high variety of initial configurations with arbitrary low entropy.

\subparagraph{Overview.}
The robots on the \conboundary{} (we call them \emph{boundary-robots}) will perform the \egtmlong\ algorithm, where robots move towards the midpoint between their two neighbors.
The advantage of this algorithm is that it is invertible (and therefore connectivity preserving) and a gathering algorithm.
The robots near to the boundary will move with the boundary-robots, one may get the impression they are being pushed to the inside by the boundary like a wave (we call them \emph{wave-robots}).
We will carefully construct the way wave-robots move such that this movement is invertible.
All other robots (called \emph{inner-robots}) do not move.
The boundary will contract, more and more inner-robots will become wave-robots and are further pushed inwards until a \neargathering\ is reached.

We split the description of the algorithm in three subsections.
In \Cref{ssec:boundary-algo} we define the boundary-robots and their algorithm.
We prove, that boundary-robots will remain a convex set during the execution of their algorithm (\Cref{lem:convex-egtm-stays-convex}) and that their algorithm is invertible (\Cref{lem:egtm-invertable}).
In \Cref{ssec:wave-algo} we define an area around the boundary-robots called \emph{Wave}.
All robots in this area are wave-robots.
We define their algorithm and prove that their movement is invertible assuming the Wave is known.
Both algorithms are combined in \Cref{ssec:main-algo}.
We depict the execution of one round in \cref{fig:execution-with-wave}.

\subsection{Boundary Algorithm}
\label{ssec:boundary-algo}

\newcommand{\boundary}[2]{\ensuremath{b^{#1}_{#2}}}
\begin{definition}[Boundary-Robots]
    \label{def:boundary-robot}
    Robots that are part of the \conboundary\ are called \emph{boundary-robots}.
    We denote the boundary robots in round $t$ by $\vec{b}^t = (\boundary{t}{0}, \cdots, \boundary{t}{k})$.
    We assume that robots are enumerated counterclockwise with $\boundary{t}{0}$ chosen arbitrarily.
\end{definition}

We define the following algorithm based on \gtmlong\ \cite{DBLP:conf/ifip10/DyniaKLH06} for boundary-robots.
Afterwards, we prove that it is invertible.

\begin{algorithm}[H]
    \caption{Boundary-algorithm: \egtmlong}
    \label{alg:egtm}
    $$\egtm(\boundary{t}{i}, \vec{b}^t) := \epsilon \cdot \frac{\boundary{t}{k-1} + \boundary{t}{k+1}}{2} + (1 - \epsilon) \cdot \boundary{t}{k}$$
\end{algorithm}
\begin{remark}
    The algorithm is in general not executable in our model.
    In general, robots cannot decide locally, whether they are boundary-robots.
    In \cref{lem:main-algo-executable}, we prove that it is executable in swarms meeting our requirements.
\end{remark}
\begin{remark}
    If not stated otherwise, we assume that $\boundary{t+1}{k} = \egtm(\boundary{t}{i}, \vec{b}^t)$.
    Theoretically, the robots on the \conboundary\ may change during the execution, depending on where other robots move.
    However, in \Cref{lem:alg-con-no-holes:conboundary-unchaged}, we prove that the set of boundary-robots does not change during the execution of our algorithm.
\end{remark}

\begin{lemma}[{name=, restate=[name=restated]lemEgtmInvertable}]
    \label{lem:egtm-invertable}
    If $\epsilon \in [0,0.5)$, \egtmlong\ is invertible for a global observer.
\end{lemma}

Note, that while \egtm\ is invertible it is only symmetry preserving when considering the neighborhood relation as part of the symmetry.

\begin{lemma}[{name=, restate=[name=restated]lemConvexEgtmStaysConvex}]
    \label{lem:convex-egtm-stays-convex}
    If $\vec{b}^t$ is convex, $\vec{b}^{t+1}$ is convex as well.
    Let $area(\vec{b}^t)$ denote the area enclosed by $\vec{b}^t$.
    Then, $area(\vec{b}^{t+1}) \subseteq area(\vec{b}^{t})$.
\end{lemma}

The proofs on \Cref{lem:egtm-invertable,lem:convex-egtm-stays-convex} can be found in \Cref{ap:contracting-wave-proofs}.

\subsection{Wave-Algorithm}
\label{ssec:wave-algo}
In this subsection we will define the set of wave-robots and construct \Cref{alg:wave-algo}.
Throughout this section we will assume that $\vec{b}^t$ (\Cref{def:boundary-robot}) is equivalent to the positions of the \conboundary\ in round $t$.
We show later in \Cref{lem:alg-con-no-holes:conboundary-unchaged} that this is indeed the case.

In \Cref{lem:convex-egtm-stays-convex} we have shown, that the area enclosed by $\vec{b}^{t+1}$ is inside the area enclosed by $\vec{b}^t$.
The goal of the wave-algorithm is, to remove all inner robots that are outside of the area of $\vec{b}^{t+1}$ inside that area such that $\vec{b}^{t+1}$ is the \conboundary\ in round $t+1$.
We first define this area formally.
Because this process is similar to a wave front that pushes the robots inwards we use the terminology wave to describe it and call the area \emph{Wave}.

\begin{definition}[Wave]
    Let $area(\vec{b}^t)$ denote the area enclosed by $\vec{b}^t$.
    We define $\wave{t}$ as $area(\vec{b}^t) \setminus area(\vec{b}^{t+1})$. 
\end{definition}

\begin{definition}[Wave-robot]
    \label{def:wave-robot}
    We call robots in $\wave{t}$ and $\wave{t+1}$ at time $t$ \emph{wave-robots}.
\end{definition}

\begin{definition}[Wave-segment]
    We cut $\wave{t}$ by cutting from $\boundary{t}{k}$ to $\boundary{t+1}{k}$ for all $1 \leq k \leq n$.
    This leaves us with $n$ \emph{Wave-Segments}.
    We call the segment with corners $\boundary{t}{k}, \boundary{t}{k+1}, \boundary{t+1}{k}, \boundary{t+1}{k+1}$ the $k$-th wave-segment of $\wave{t}$ or $\waveseg{t}{k}$.
\end{definition}

We will design \Cref{alg:wave-algo} such that all robots from $\waveseg{t}{k}$ move into $\waveseg{t+1}{k}$.
The robots in $\waveseg{t+1}{k}$ move inside their segment as well, to prevent collisions with incoming robots.

\subparagraph{Preliminary Statements}
To use wave-segments as a base for our algorithm, we need to make sure that they partition the robots unambiguously. 
In the following we prove that segments do not overlap and are not twisted.
But there are configurations of $\vec{b}^t$ (in case of collinear robots in $\vec{b}^t$) where the quadrilateral is degenerated, i.e. partly without area or just a line without any area.
We prove that a segment in $\wave{t}$ will not become more degenerated in $\wave{t+1}$.

\begin{corollary}[{name=, restate=[name=restated]corWaveSegmentsNotOverlappingNotTwisted}]
    \label{lem:wave-segments-not-overlapping-not-twisted}
    Let $\vec{b}^t$ be a convex set of robots.
    The segments of $\wave{t}$ are not twisted quadrilaterals, i.e. two sided do not intersect on a single point, and do not overlap.
\end{corollary}

\begin{corollary}[{name=, restate=[name=restated]corNonDegeneratedStaysNonDegenerated}]
    \label{lem:non-degenerated-stays-non-degenerated}
    We call a quadrilateral where all four corners are not collinear \emph{non-degenerated}, a \emph{partially degenerated} quadrilateral has 3 collinear corners and a \emph{fully degenerated} has 4 collinear corners.
    \begin{enumerate}
        \item If $\waveseg{t}{k}$ is a non-degenerated quadrilateral,
        $\waveseg{t+1}{k}$ is also a non-degenerated quadrilateral.
        \item If $\waveseg{t}{k}$ be a partially degenerated segment,
        $\waveseg{t+1}{k}$ is a non-degenerated segment.
    \end{enumerate}
\end{corollary}

\begin{corollary}[{name=, restate=[name=restated]corDistanceToBoundaryInSegment}]
    \label{lem:distance-to-boundary-in-segment}
    Robots in \waveseg{t}{k} and \waveseg{t+1}{k} have a distance of $\leq 1 + \nicefrac{\epsilon^2}{2}$ to \boundary{t}{k} and \boundary{t}{k+1}.
\end{corollary}

The proofs of \Cref{lem:wave-segments-not-overlapping-not-twisted,lem:non-degenerated-stays-non-degenerated,lem:distance-to-boundary-in-segment} can be found in \Cref{ap:contracting-wave-proofs}.

\subparagraph{Definition of the Algorithm}

The algorithm uses a bijective mapping from the area $\waveseg{t}{k} \cup \waveseg{t+1}{k}$ onto $\waveseg{t+1}{k}$.
For this mapping we define a normalized two-dimensional coordinate system inside each segment from ranging from $(0,0)$ to $(1,1)$ (\Cref{def:delta-line}).
In rectangular segments these are simple vertical and horizontal lines and the $x$- and $y$ unit-distance is scaled accordingly, for other shapes the lines are adjusted to follow the boundaries of the shape.
\Cref{alg:wave-algo} is a simple mapping between the areas of the segments based on the normalized coordinates.
Wave robots from \waveseg{t}{k} move inside the outer half of \waveseg{t+1}{k} while the robots inside \waveseg{t+1}{k} move into the inner half of \waveseg{i+1}{k}.
See \Cref{fig:execution-with-wave} for an example.
In the end we prove that the wave algorithm is invertible (\Cref{lem:wave-algo-bijective}).

\begin{definition}[normalized coordinate-system inside a quadrilateral]
    \label{def:delta-line}
    Let $A,B,C,D$ be the corners of the quadrilateral, we denote the line-segments between the corners by $\overline{AB}$.
    For $x \in [0,1]$ we define $(x, 0) := A + x \cdot (D - A)$ and $(x, 1) := B + x \cdot (C - B)$ (in particular $(0,0) = A; (1,0) = B; (1,1) = C$ and $(0,1) = D$).
    If the quadrilateral is convex, for $y \in [0,1]$ we equally distribute $(x,y)$ on the straight line between $(x,0)$ and $(x,1)$.
    For non-convex quadrilaterals we assume w.l.o.g. that $C$ is the concave corner.
    For $x \in [0,1]$ we connect $(x,0)$ and $(x,1)$ with a parallel to $\overline{AB}$ starting at $(x,0)$ and a parallel to $\overline{CD}$ starting at $(x,1)$.
    For $y \in [0,1]$ we equally distribute $(x,y)$ on this connection.
    See \Cref{fig:wave-coordinate-sys} for an example.

    For wave-segment $\waveseg{t}{k}$ we define $A = \boundary{t}{k}, B = \boundary{t}{k+1}, C = \boundary{t+1}{k+1}$ and $D = \boundary{t+1}{k}$.
    To denote the position $(x,y)$ inside $\waveseg{t}{k}$ we use the notation $\waveseg{t}{k}(x,y)$.
\end{definition}

\begin{algorithm}[H]
    \caption{Wave-Algorithm}
    \label{alg:wave-algo}
    This algorithm gets the positions of boundary-robots $\vec{b}^t$ as input.
    It is used to compute \wave{t} and \wave{t+1}. 
    The algorithm then determines whether $z^t_i$ is in $\waveseg{t}{k}$ or $\waveseg{t+1}{k}$ for some $k$ and the coordinates $x$ and $y$ according to \cref{def:delta-line}.
    \[
    \wavealgo(z^t_i, \vec{b}^t)
    =
\begin{cases}
    \waveseg{t+1}{k}(x/2,y),& \text{if } z^t_i = \waveseg{t}{k}(x,y)\\
    \waveseg{t+1}{k}(1/2 + x/2,y),& \text{if } z^t_i = \waveseg{t+1}{k}(x,y)
\end{cases}
\]
\end{algorithm}

\begin{lemma}
    \label{lem:wave-algo-bijective}
    Assuming $\vec{b}^t$ and the wave-robots in round $t$ are fixed and known.
    After executing \Cref{alg:wave-algo} with all wave-robots, we can compute the positions of the wave-robots in round $t$.
\end{lemma}
\begin{proof}
    All robots from $\waveseg{t}{k}$ and $\waveseg{t+1}{k}$ moved inside $\waveseg{t+1}{k}$ for $0 \leq k < |\vec{b}^t|$.
    The segments do not overlap (\Cref{lem:wave-segments-not-overlapping-not-twisted}), therefore the movement is unambiguous.
    If $\waveseg{t}{k}$ is non-degenerated, then $\waveseg{t+1}{k}$ is also non-degenerated (\Cref{lem:non-degenerated-stays-non-degenerated}).
    In both segments, the coordinates are unambiguous.
    Therefore, \Cref{def:delta-line} yields a bijective mapping.

    If $\waveseg{t}{k}$ is partly degenerated, there may exist $(x,y) \neq (x',y')$ with $\waveseg{t}{k}(x,y) = \waveseg{t}{k}(x',y')$.
    But $\waveseg{t+1}{k}$ is non-degenerated (\Cref{lem:non-degenerated-stays-non-degenerated}), therefore no two robots move onto the same positions in $\waveseg{t+1}{k}$.
    Therefore, the origins for all robots in unambiguous.
    If $\waveseg{t+1}{k}(x,y)$ is fully-degenerated (is only a line), it follows from \Cref{lem:non-degenerated-stays-non-degenerated} that $\waveseg{t}{k}(x,y)$ must be fully-degenerated as well.
    But because both segments are only a line without area in this case, the mapping is bijective.
\end{proof}

\subsection{Main Algorithm}
\label{ssec:main-algo}
We first define \emph{inner-robots}, that are the remaining robots beside boundary- and wave-robots.
Afterwards we state the main algorithm formally and prove its correctness as well as that it is symmetry preserving.

\begin{definition}[inner-robots]
    \label{def:inner-robot}
    We call robots in $\wave{k}, k > t+1$ at time $t$ \emph{inner-robots}.
\end{definition}

\begin{algorithm}
    \caption{Main Algorithm}
    \label{alg:near-gath-convex-no-holes}

    Based on $\z^t$ the positions of boundary-robots $\vec{b}^t$ can be observed.
    $\vec{b}^t$ is used to compute \wave{i} and \wave{i+1} to determine wave-robots and inner-robots.
    \[
        f(z^t_i, \z^t)= 
    \begin{cases}
        \egtm(z^t_i, \vec{b}^t)& \text{if } z^t_i \text{ is a boundary-robot (\Cref{def:boundary-robot})},\\
        \wavealgo(z^t_i, \vec{b}^t)& \text{if } z^t_i \text{ is a wave-robot (\Cref{def:wave-robot})},\\
        z^t_i& \text{if } z^t_i \text{ is an inner-robot (\Cref{def:inner-robot})}.
    \end{cases}
    \]
\end{algorithm}
\begin{remark}
    To execute \Cref{alg:near-gath-convex-no-holes} as it is written, robots must observe $\vec{b}^t$.
    They are not able to observe $\vec{b}^t$ fully because of the limited visibility in out model.
    But we prove in \Cref{lem:main-algo-executable} that they are able to observe the locally relevant part of $\vec{b}^t$ to compute \Cref{alg:near-gath-convex-no-holes} locally.
\end{remark}

\begin{figure}
    \begin{subfigure}{0.65\linewidth}
        \includegraphics[page=2]{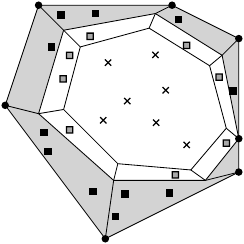}
        \hfill
        \includegraphics[page=3]{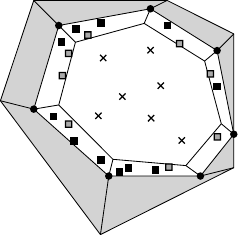}
        \label{fig:execution-with-wave}
        \caption{
            Execution of \Cref{alg:near-gath-convex-no-holes} one round. 
            $\bullet$ are boundary-robots, $\square$ are wave-robots and $\times$ are inner-robots in round $t$. 
            \wave{t} (gray) and \wave{t+1} are partitioned in segments.
        }
    \end{subfigure}
    \hfill
    \begin{subfigure}{0.3\linewidth}
        \includegraphics[page=6,width=\linewidth]{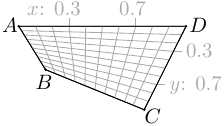}

        \caption{Shows a grid inside a segment according to \Cref{def:delta-line}.}
        \label{fig:wave-coordinate-sys}
    \end{subfigure}
\end{figure}

\begin{lemma}
    \label{lem:alg-con-no-holes:conboundary-unchaged}
    Let $\vec{b}^t$ be convex and the \conboundary\ of a swarm without $2.24$-holes.
    Let $\vec{b}^{t+1}$ be the \conboundary\ after executing \Cref{alg:near-gath-convex-no-holes} one round and let $\egtm(\vec{b}^t)$ denote the positions of the boundary-robots in round $t$ after executing the algorithm.
    $\vec{b}^{t+1} = \egtm(\vec{b}^t)$.
\end{lemma}
\begin{proof}
    From \Cref{lem:convex-egtm-stays-convex} we know, that $\egtm(\vec{b}^t)$ is convex.
    Neighboring positions in $\egtm(\vec{b}^t)$ have distance $\leq 1$.
    Therefore, $\egtm(\vec{b}^t)$ would be the \conboundary\ of the new configuration, if all other robots are within $area(\egtm(\vec{b}^t))$.
    We defined $\wave{t}$ as $area(\vec{b}^t) \setminus area(\egtm(\vec{b}^t))$.
    \Cref{alg:wave-algo} is designed such that all robots from $\wave{t}$ move into $\wave{t+1}$, therefore moving inside $area(\egtm(\vec{b}^t))$.
    Only robots from coordinates $\waveseg{t}{k}(x,0)$ move onto coordinates $\waveseg{t+1}{k}(x,0)$ (see \Cref{alg:wave-algo}).
    On coordinates $\waveseg{t}{k}(x,0)$ are by definition only boundary-robots.
    Therefore, $\egtm(\vec{b}^t)$ is the \conboundary\ of the swarm in round $t+1$.
\end{proof}

\begin{lemma}
    \label{lem:main-algo-executable}
    In a configuration with a convex \conboundary\ and no $2.24$-holes, \Cref{alg:near-gath-convex-no-holes} \textbf{is executable} for \oblot/-robots with a viewing range of $2+\sqrt{2}$.
\end{lemma}
\begin{proof}
    \textbf{Decide robot state locally (boundary, wave, inner).}
    With no $2.24$-holes, a robot is on the \conboundary\ if it is on the border of the convex hull of its $2.24$-surrounding.
    This can be determined with a viewing range of 2.24.
    A robot $r$ inside $\waveseg{t}{k}$ or $\waveseg{t+1}{k}$ has a distance $\leq 1 + \nicefrac{\epsilon^2}{2} < 1.12$ to $\boundary{t}{k}$ and $\boundary{t}{k+1}$ (see \Cref{lem:distance-to-boundary-in-segment}).
    Therefore, $r$ with a viewing range of $1.12 + 2.24 < 2 + \sqrt{2}$ can determine, whether the robot on $\boundary{t}{k}$ is a boundary-robot.
    To compute $\waveseg{t}{k}$ and $\waveseg{t+1}{k}$ robots on $\boundary{t}{k-2}, \cdots, \boundary{t}{k+3}$ must be observed.
    To $\boundary{t}{k}$ and $\boundary{t}{k+1}$ these have a maximal distance of up to $2$.
    Because the robot on \boundary{t}{k} could already be identified, it is known where the outside of the \conboundary\ is.
    We know that the \conboundary\ is connected with a distance $\leq 1$.
    Therefore, the $1$-surrounding around a known boundary-robot is sufficient to identify its next neighbor along the boundary.
    With a viewing range of $2+\sqrt{2}$ the $1$-surrounding of $\boundary{t}{k}$ and $\boundary{t}{k+1}$ as well as from $\boundary{t}{k-1}$ and $\boundary{t}{k+2}$ is observable.
    This allows to identify robots on $\boundary{t}{k-2}, \cdots, \boundary{t}{k+3}$.
    Therefor, each robot in $\waveseg{t}{k}$ and $\waveseg{t+1}{k}$ can identify, that it is in the mentioned segment (wave-robots).
    All robots $r$ not in these segments can either not find sufficient many boundary-robots or can compute that they are not in segments of $\wave{t}$ or $\wave{t+1}$ adjacent to the observed boundary-robots.
    Both is sufficient to decide, that $r$ is an inner robot.

    \textbf{Compute the algorithm locally.}
    For boundary robots (\egtmlong) and inner robots (do not move) this is trivial.
    The wave-segments are not overlapping.
    Target positions are computed based on the segment the robot is in.
    Therefore, robots not on the borders of a segment can unambiguously compute their target positions.
    Robots on the border of a segment will compute for both segments the exact target position, therefor it is unambiguous as well.
    The movement distance is $\leq 1$, therefore the computed move can be executed.
\end{proof}

\begin{lemma}
    \label{lem:alg-con-no-holes:collisionfree}
    In a configuration with a convex \conboundary\, \Cref{alg:near-gath-convex-no-holes} \textbf{does not lead to collisions}.
\end{lemma}
\begin{proof}
    We consider round $t$.
    Robot within $\wave{t'}, t' \geq t +2$ (inner-robots) cannot have collisions with each other, because they do not move.
    They can also have no collisions with other robots, because boundary- and wave-robots move within $\wave{t}$ and $\wave{t+1}$.
    Boundary robots are essentially wave robots with the special case that they are on coordinates $\waveseg{t}{k}(x,0)$.
    From the proof of \Cref{lem:wave-algo-bijective} follows, that \Cref{alg:wave-algo} is collision-free.
\end{proof}

\begin{lemma}
    \label{lem:alg-con-no-holes:no-holes-created}
    In a configuration with convex \conboundary\ and initially no $1$-holes, \Cref{alg:near-gath-convex-no-holes} \textbf{does not create $2.24$-holes}.
\end{lemma}
\begin{proof}
    All robots not inside $\wave{t}$ have been inner robots until now and have not moved at all.
    Therefore, holes that have no overlap with $\wave{t}$ must have existed initially and can only have diameter $<1$.
    Let us consider a $2.24$-hole that overlaps with $\wave{t}$.
    \begin{figure}
        \begin{minipage}[c]{0.35\textwidth}
            \includegraphics[width=\textwidth]{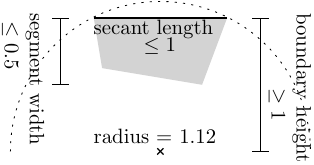}
        \end{minipage}\hfill
        \begin{minipage}[c]{0.5\textwidth}
            \caption{Distance between $\waveseg{t}{k}$ (gray) and the midpoint of a $2.24$-hole.}
            \label{fig:hole-wave-distance}
        \end{minipage}
      \end{figure}
    See the construction of $\waveseg{t}{k}$ that overlaps with a $2.24$-hole in\Cref{fig:hole-wave-distance}.
    $\waveseg{t}{k}$ cannot lie completely within the hole, because that would mean the boundary robots are inside the hole.
    But the boundary of $\waveseg{t}{k}$ can be a secant with length $\leq 1$ of the hole.
    The height of such a secant of a circle with radius $1.12$ is $\geq 1$.
    A wave-segment can have a maximal width of 0.5, therefore all points of $\waveseg{t}{k}$ have a distance $\geq 0.5$ to the midpoint of the hole.
    Therefore, $1$-hole with the same midpoint does not overlap with $\wave{t}$ and must have existed initially.
\end{proof}

\begin{lemma}
    \label{lem:alg-con-no-holes:invertible}
    In a configuration with a convex \conboundary, \Cref{alg:near-gath-convex-no-holes} \textbf{is locally invertible}.
\end{lemma}
\begin{proof}
    Let us assume $\vec{b}^t$ is convex and the \conboundary\ of the swarm in round $t$.
    For the initial configuration this is true by definition.
    We will show, that we can compute the swarm in round $t$ from round $t+1$.
    From \Cref{lem:convex-egtm-stays-convex} we know that $\vec{b}^{t+1}$ is convex.
    From \Cref{lem:alg-con-no-holes:conboundary-unchaged} we know, that $\vec{b}^{t+1}$ is the \conboundary\ in round $t+1$.
    The \conboundary\ can easily be identified by a global observer, therefore $\vec{b}^{t+1}$ is known.
    By \Cref{lem:egtm-invertable} we know that \egtm\ is invertible.
    Therefore, we can compute $\vec{b}^t$ and $\wave{t}$ as well as $\wave{t+1}$.
    We know that all robots that were in round $t$ not in $\wave{t}$ or $\wave{t+1}$ are inner-robots and have not moved (\Cref{def:inner-robot}).
    Therefore, all robots in $\wave{t+1}$ in round $t+1$ must have been wave- or boundary-robots in round $t$.
    Because all wave- and boundary-robots in round $t$ move into $\wave{t+1}$ the robots in $\wave{t+1}$ in round $t+1$ are the complete set of wave- and boundary-robots in round $t+1$.
    The boundary-robots are already identified, therefore we know the set of wave-robots.
    This allows us to apply \Cref{lem:wave-algo-bijective} to compute the positions of the wave-robots in round $t$.
    All other robots are inner robots, therefore they do not move in round $t$.
\end{proof}

\begin{theorem}
    We assume robots according the \oblot/ model and \fsync\ scheduler with a viewing range of $2 + \sqrt{2}$ in a swarm with a convex \conboundary\ and no $1$-holes.
    \Cref{alg:near-gath-convex-no-holes} leads to \neargathering\ and does not change the symmetry.
\end{theorem}
\begin{proof}
    The algorithm is executable in the assumed model (\Cref{lem:main-algo-executable}).
    From \cite{DBLP:conf/ifip10/DyniaKLH06} we know, that \egtmlong\ converges towards gathering.
    From \Cref{lem:alg-con-no-holes:conboundary-unchaged} we know, that $\vec{b}^{t+1} = \egtm(\vec{b}^t)$.
    Therefore, $\vec{b}^t$ will converge towards gathering for $t \leftarrow \infty$.
    Because $\vec{b}^t$ is the \conboundary\, all robots are within $area(\vec{b}^t)$.
    At the same time, each robot keeps its own position (\Cref{lem:alg-con-no-holes:collisionfree}).
    Therefore, eventually a \neargathering\ will be reached.
    Because \Cref{alg:near-gath-convex-no-holes} is locally invertible (\Cref{lem:alg-con-no-holes:invertible}) we can follow from \Cref{thr:main-symmetries-math} that the symmetries are not changed.
\end{proof}

\begin{remark}
    \label{remark:egtm-always-near-gathering}
    The algorithm presented above works for general swarms.
    Even if the \conboundary\ is not convex, the robots at the convex hull of the swarm can always be detected as boundary-robots and move inwards, such that the convex hull shrinks monotonically. 
    There may be movements at other parts of the swarm, especially if the swarm contains holes, but \egtm\ does never result in connectivity loss.
    For situations where a robot is only connected to one other robots (e.g. robots form a line) \egtm\ must be slightly altered (a robot considers one neighbor as its left and right neighbor at the same time in this situations).
    While a \neargathering{} is always reached, the symmetry preservation cannot be guarantied in general.
\end{remark}
\section{Conclusion \& Discussion}%
\label{sec:conclusion}

In our paper we partially solved the \neargathering\ problem with symmetry preservation.
On one hand, we gave a variant of \gta\ that preserves symmetry for all initial configurations and even reaches \neargathering for certain configurations.
On the other hand, \Cref{alg:near-gath-convex-no-holes} solves \neargathering\ for all configurations but preserves symmetry only for subset of configurations.
This opens a few questions.

\textbf{Can we loosen the restrictions of \Cref{alg:near-gath-convex-no-holes}?}
A central part of our analysis is that the robots on the \conboundary\ perform the same linear function during the execution.
This is only the case, if the set of boundary-robots never changes.
If we allow for a non-convex \conboundary, 
robots move towards the outside of the swarm in concave sections.
So, eventually two robots of the \conboundary\ that initially have a distance $> 1$ will reach a distance $\leq 1$.
One of these robots is longer part of the \conboundary\ in the next round.
One could fix this by introducing memory, but with memory symmetry preserving \neargathering\ is already solved \cite{DBLP:conf/sirocco/YamauchiY13}.
Another possible fix is to not move robots on concave parts of the \conboundary.
But eventually, these parts become convex and robots start to move according to \egtm\ which, again, changes the linear function.
The second restriction are holes.
Robots at the boundary of large holes cannot distinguish their location locally from the \conboundary\ (not even with memory).
Therefore, they will start with a \neargathering\ that eventually leads to a configuration, where no hole exists anymore.
But this configuration can have another pre-image where the inner robots already had the position of the not-anymore-hole.
Therefore, the algorithm is not invertible in this case.

\textbf{Can we create a framework for a more general strategy from \Cref{alg:near-gath-convex-no-holes}?}
The proofs in \Cref{sec:protocol:contractingwaves} can be generalized for a class of strategies.
Our core idea is to split the robots into layers (e.g. boundary, wave, inner).
The outermost layer performs a \neargathering\ algorithm that is symmetry preserving.
All other layers perform algorithms, such that they stay inside the outermost layer.
The outermost layer is always distinguishable and invertible.
The inner layers are distinguishable and invertible, if the outer layers are known.
The advantage of these class of algorithms is, that you can reduce the problem of an invertible \neargathering\ algorithm to a restricted set of robots.
However, even for a restricted set of robots like those on the \conboundary\ it turns out to be a major challenge to find a connectivity preserving algorithm.

\textbf{For which class of configurations does \gta\ lead to \neargathering?}
The configurations must contain evenly spread robots.
One example is a square (or triangle or hexagon) that is filled with a regular grid with distance $<$ viewing range.
Also perturbations of such configurations lead to \neargathering.

\textbf{When do \neargathering\ algorithms from \cite{DBLP:conf/opodis/CastenowH0KKH22} increase symmetry?}
In \cite{DBLP:conf/opodis/CastenowH0KKH22}, a class of \neargathering\ functions is introduced.
These functions do not in general preserve the symmetry.
A simple example is the \GtC\ algorithm where robots move onto the center of the smallest enclosing circle of their local neighborhood.
Let us assume a configuration $\z^t$ with symmetry.
After performing \GtC\ the new configuration $\z^{t+1}$ must be symmetric as well.
The smallest enclosing circle is dependent only on robots exactly on the circle.
We fix all robots that are on these circles and perturbate all other robots such that the configuration is asymmetric.
Let this be configuration $\breve{\z}^t$.
After performing \GtC\ the new configuration is $\breve{\z}^{t+1}$.
Because the smallest enclosing circles have not changed, the robots in $\breve{\z}^t$ move onto the same positions as the robots in $\z^t$.
Therefore, with $\breve{\z}^{t+1} = \z^{t+1}$ we introduced symmetries.

\textbf{Conclusion}
We presented in \Cref{sec:protocol:averaging} the first known non-trivial algorithm for local robots that preserves symmetry.
This in itself is no small achievement because the local capabilities (especially limited visibility) make it impossible for robots to observe the current symmetry and to act accordingly.
Additionally, we presented a \neargathering\ algorithm that preserves the symmetry for a large subset of initial configurations.
The algorithm is based on \egtm, a linear function that leads to the gathering of swarms connected in a chain.
The advantage of a linear function is that we can analyze the invertibility using well-known properties from linear algebra.
Other gathering algorithms like \GtC\ contain conditions, what robots in the viewing range to consider for the computation (\textsc{GtC} ignores all but three robots).
Such properties make it hard to analyze their invertibility to apply our analysis method.
As we have argued in the discussion, the initial configuration limitation stems from the problem of getting a consistent chain of robots (in our case, the \conboundary) that executes \egtm\ while having no memory and no global knowledge. 
\bibliography{main_brtrdb,references}

\appendix
\numberwithin{equation}{section}
\renewcommand{\theequation}{\thesection.\arabic{equation}}

\section{Dynamical Systems Proofs for \Cref{sec:preliminaries}}%
\label{ap:math-proofs}

In this appendix, we provide some additional details on the mathematical dynamical systems and symmetries perspective and, in particular, prove \Cref{thr:main-symmetries-math}.

\subsection{Symmetries of a protocol}
We begin by investigating the interplay of the symmetries defined in \cref{sec:preliminaries} and the dynamics induced by a protocol. The robots do not know the global coordinate system that we choose as the external observer. Hence, their computation and movement is insensitive to rotations of the global coordinates in the sense that collectively rotating all robots positions by $\rho$ causes each robot to  move to a rotated target point in each round.
\begin{remark}
	\label{remk:translation}
	In principle, the same holds true for translations of the global coordinate systems. However, since we fixed the origin, we typically omit these transformations without loss of generality.
\end{remark}

On the other hand, recall that the robots are indistinguishable. That means if a robot observes another robot in a certain position, it does not know which label this robot has. More precisely, the computations and movement of every single robot depend on the set of the other robots' positions rather than their ordered tuple. Once again, this can be recast by stating that computation and movement are insensitive to arbitrary permutations of all robots positions.

Summarizing these observations and reformulating them in terms of the function governing the dynamics of all robots $f$, we obtain
\begin{lemma}
	\label{lem:invariance}
	Let $\eta\in\R^2$ and $\vzeta=(\zeta_1,\dotsc,\zeta_n)^\trans\in\R^{2n}$ be arbitrary. The function governing the dynamics of all robots $f$ has the following \emph{symmetry properties}:
	\begin{enumerate}[label=(\roman*)]
		\item $f(\rho \eta; \rho \zeta_1, \dotsc, \rho \zeta_n) = \rho f(\eta; \mathbf{\zeta})$ for all rotations $\rho\colon\R^2\to\R^2$;
		\item $f(\eta; \zeta_{\kappa(1)}, \dotsc, \rho \zeta_{\kappa(n)}) = f(\eta; \mathbf{\zeta})$ for all permutations $\kappa\colon\set{1,\dotsc,n}\to\set{1,\dotsc,n}$.
	\end{enumerate}

	These can be restated using matrices \eqref{eq:symm-mat} as
	\begin{enumerate}[label=(\roman*)]
		\item $f(\rho \eta; M_\rho\mathbf{\zeta}) = \rho f(\eta; \mathbf{\zeta})$ for all rotations $\rho\colon\R^2\to\R^2$;
		\item $f(\eta; M_\kappa\mathbf{\zeta}) = f(\eta; \mathbf{\zeta})$ for all permutations $\kappa\colon\set{1,\dotsc,n}\to\set{1,\dotsc,n}$.
	\end{enumerate}
\end{lemma}

From the global perspective, considering the collective evolution of the entire formation, these symmetry properties imply that the evolution is insensitive to arbitrary rotations and arbitrary permutations of the robots. More precisely, it does not matter if first all robots compute/move and then rotate/permute or do it the other way around. More precisely, we may even replace rotate/permute by rotate and permute which gives us the combined transformations in $G$.
\begin{proposition}
	\label{prop:equiv}
	The evolution function $F$ is symmetric---or \emph{equivariant}---with respect to all potential symmetries of formations $M_\kappa M_\rho \in G$:
	\begin{equation}
		\label{eq:equiv}
		F \circ (M_\kappa M_\rho) = (M_\kappa M_\rho)\circ F.
	\end{equation}
\end{proposition}
\begin{proof}
	We claim that it suffices to prove that $F$ is equivariant as in \eqref{eq:equiv} with respect to $M_\kappa$ for all $\kappa$ and $M_\rho$ for all $\rho$ individually to prove the statement. In fact, if this holds true we immediately see
	\begin{align*}
		F \circ (M_\kappa M_\rho) 	&= F \circ (M_\kappa \circ M_\rho)\\
									&= (F \circ M_\kappa) \circ M_\rho \\
									&= (M_\kappa \circ F) \circ M_\rho \\
									&= M_\kappa \circ (F \circ M_\rho) \\
									&= M_\kappa \circ (M_\rho \circ F) \\
									&= (M_\kappa M_\rho) \circ F.
	\end{align*}

	Hence, we prove the claim using the symmetry properties in \Cref{lem:invariance}. To that end let $\vzeta=(\zeta_1,\dotsc,\zeta_n)^\trans\in\R^{2n}$ be an arbitrary point in configuration space, $\kappa\colon\set{1,\dotsc,n}\to\set{1,\dotsc,n}$ an arbitrary permutation, and $\rho\colon\R^2\to\R^2$ an arbitrary rotation. Then, we compute
	\begin{equation}
		\label{eq:sym-kappa-F}
		\begin{split}
			(F\circ M_\kappa)(\vzeta)	&= F(M_\kappa \vzeta) \\
										&= \begin{pmatrix} f((M_\kappa \vzeta)_1, M_\kappa \vzeta) \\ \vdots \\ f((M_\kappa \vzeta)_n, M_\kappa \vzeta) \end{pmatrix} \\
										&= \begin{pmatrix} f(\zeta_{\kappa(1)}, M_\kappa \vzeta) \\ \vdots \\ f(\zeta_{\kappa(n)}, M_\kappa \vzeta) \end{pmatrix}  \\
										&= \begin{pmatrix} f(\zeta_{\kappa(1)}, \vzeta) \\ \vdots \\ f(\zeta_{\kappa(n)}, \vzeta) \end{pmatrix} \\
										&= \begin{pmatrix} F(\vzeta)_{\kappa(1)} \\ \vdots \\ F(\vzeta)_{\kappa(n)} \end{pmatrix} \\
										&= M_\kappa F(\vzeta) \\
										&= (M_\kappa \circ F)(\vzeta),
		\end{split}
	\end{equation}
	where the fourth equality holds due to \Cref{lem:invariance}. Similarly, we obtain
	\begin{equation}
		\label{eq:sym-rho-F}
		\begin{split}
			(F\circ M_\rho)(\vzeta)	&= F(M_\rho \vzeta) \\
									&= \begin{pmatrix} f((M_\rho \vzeta)_1, M_\rho\vzeta) \\ \vdots \\ f((M_\rho \vzeta)_n, M_\rho\vzeta) \end{pmatrix} \\
									&= \begin{pmatrix} f(\rho \zeta_1, M_\rho\vzeta) \\ \vdots \\ f(\rho \zeta_n, M_\rho\vzeta) \end{pmatrix} \\
									&= \begin{pmatrix} \rho f(\zeta_1, \vzeta) \\ \vdots \\ \rho f(\zeta_n, \vzeta) \end{pmatrix} \\
									&= M_\rho F(\vzeta) \\
									&= (M_\rho F)(\vzeta)
		\end{split}
	\end{equation}
	again using \Cref{lem:invariance}. This completes the proof of the claim.
\end{proof}
\noindent
By the previous proposition, $F$ commutes with all elements of $G$. Hence, whenever we refer to an arbitrary symmetry without the need to specify rotation and permutation separately, we use $M, M',\dotsc \in G$ from now on.

\subsection{Proof of \Cref{thr:main-symmetries-math}}%
\label{sec:proofthm1}

\Cref{prop:equiv} places our mathematical framework in the context of \emph{equivariant dynamics}, for which there exists a well developed theory to investigate the interplay of dynamics and symmetries (e.g. \cite{Golubitsky.1988,Chossat.2000}). It allows us to prove \Cref{thr:main-symmetries-math}. We do so in the following three propositions, which combined give the statement of the theorem.

\begin{proposition}
    \label{prop:ap-preserve}
	Consider the dynamics of a configuration according to an arbitrary protocol \eqref{eq:general_configuration}. Then the configuration after one round cannot have fewer symmetries than the initial one: $G_{\z} \subset G_{\z^+}$.
\end{proposition}
\begin{proof}
	Let $\z\in\R^{2n}$ be some arbitrary configuration and $M\in G_{\z}$. Consider the evolution $\z^+=F(\z)$ in one round. Then
	\begin{equation*}
		M\z^+=MF(\z)=F(M\z)=F(\z)=\z^+,
	\end{equation*}
	where we have exploited the symmetry of $F$ (\Cref{prop:equiv}) as well as the fact that $M$ leaves $\z$ unchanged. In particular, this implies $M\in G_{\z^+}$ proving the statement.
\end{proof}

In a very similar manner we may prove that a configuration cannot gain any symmetries during the temporal evolution. This statement, however, is only true in general if the evolution function $F$ is invertible.
\begin{proposition}
    \label{prop:ap-nogain}
	Consider the dynamics of a configuration according to an arbitrary protocol \eqref{eq:general_configuration}. Assume that the evolution function $F\colon\R^{2n}\to\R^{2n}$ is invertible. Then the configuration after one round cannot have more symmetries than the initial one. In particular, $G_{\z} = G_{\z^+}$.
\end{proposition}
\begin{proof}
	The setup for the proof is the same as in the previous. Let $\z\in\R^{2n}$ be some arbitrary configuration. Consider the evolution $\z^+=F(\z)$ in one round. By assumption, $F$ is invertible and we may restate $\z=F^{-1}(\z^+)$. It can readily be seen that the inverse $F^{-1}$ has the same symmetry properties as $F$:
	\[ F \circ M = M \circ F \iff F \circ M \circ F^{-1} = M \iff M \circ F^{-1} = F^{-1} \circ M \]
	for any $M\in G$.

	In particular, for $M\in G_{\z^+}$ we may apply \Cref{prop:ap-preserve} to $F^{-1}$ to obtain $M\in G_{\z}$. This implies $G_{\z^+} \subset G_{\z}$, which in combination with \Cref{prop:ap-preserve} proves the claim.
\end{proof}

The previous proposition requires the existence of a \emph{global} inverse $F^{-1}$ to $F$. However, the weaker notion of \emph{local} invertibility is sufficiently strong to draw the conclusions of \Cref{prop:ap-nogain}.
\begin{proposition}
	Consider the dynamics of a configuration according to an arbitrary protocol \eqref{eq:general_configuration}. Assume that the evolution function $F\colon\R^{2n}\to\R^{2n}$ is \emph{locally} invertible. Then, the configuration after one round cannot have more symmetries than the initial one. In particular, $G_{\z} = G_{\z^+}$.
\end{proposition}
\begin{proof}
	As we have a local inverse for every configuration $\z\in\R^{2n}$, we may compare $G_{\z}$ and $G_{\z^+}$ using the local inverse as before. To that end, let $\z\in\R^{2n}$ be an arbitrary configuration, $\z^+=F(\z)$, $F_{\z}^{-1}$ the local inverse, and $M\in G_{\z^+}$. Then
	\begin{align*}
		\z	&= F_{\z}^{-1}(\z^+) \\
			&= F_{\z}^{-1}(M\z^+) \\
			&= F_{\z}^{-1}(MF(F_{\z}^{-1}(\z^+))) \\
			&= F_{\z}^{-1}(F(MF_{\z}^{-1}(\z^+))) \\
			&= MF_{\z}^{-1}(\z^+) = M\z.
	\end{align*}
	All applications of the local inverse are well-defined, as $M$ leaves $\z^+$ unchanged. In particular, we have shown that $M\in G_{\z}$ proving the necessary inclusion as in \Cref{prop:ap-nogain}.
\end{proof}

\begin{remark}
    One readily confirms, that the set of all potential symmetries $G$ has the algebraic structure of a group, i.e., it contains the identity, is closed under products, and contains the inverse of every element---fully justifying the reference to equivariant dynamics. Furthermore, the subset of symmetries of a configuration $G_\z$ is a subgroup, i.e., a subset that is a group itself. Consequently, we may refer to $G$ as the \emph{symmetry group} and to $G_\z$ as the \emph{isotropy subgroup}. One may further compute that all matrices in $G$ have determinant $1$. Hence, $G$ is a subgroup of the special orthogonal group $SO(2n)$ consisting of all rotations of $R^{2n}$.
\end{remark}
\begin{remark}
    The equivariant dynamical systems formalism immediately offers simple means to deduce features of the collective dynamics. In fact, any subspace of $R^{2n}$ that is pointwise mapped to itself by any of the potential symmetries in $G$ cannot be left by the collective dynamics. This can for example be used to show the invariance of certain formations. However, since this is not the focus of this paper, we omit any details at this point. 
\end{remark}
\section{Invertibility of Protocol 1}%
\label{ap:protocol1}

Here we prove \cref{lem:prot1:invertibleF}, namely that our first protocol (\cref{sec:protocol:averaging}) satisfies the assumptions of \Cref{thr:main-symmetries-math}.
In fact, we need to confirm, that the induced evolution function is indeed locally invertible.
However, as this proof is technical and tedious but not very enlightening, we omit most of the technicalities and provide a sketch instead.
The main ingredients of the proof are
\begin{itemize}
    \item the inverse function theorem (e.g., \cite{Hirsch.1988}), which states that that a continuously differentiable function is locally invertible at every point where its Jacobian is an invertible matrix, and
    \item the Gershgorin circle theorem \cite{Gersgorin.1931}, which states that all eigenvalues of a matrix $A=(a_{i,j})_{i,j=1}^n$ are contained in the union of the circles $\{z \in \C \mid |z-a_{i,i}| \le \sum_{j=1, j\ne i}^n|a_{i,j}|\}$.
\end{itemize}

We first fix some notation. Recall from \eqref{eq:ex_dynamics} that
\[ f(z_i, \z) = z_i + \frac{\varepsilon}{n} \sum_{i=j}^n b(\|z_i-z_j\|^2)(z_j-z_i), \]
which is smooth---and thus continuously differentiable in particular---by construction. Furthermore, it is a two-dimensional expression. The collection of these expressions for $i=1,\dotsc, n$ is the evolution function $F$. To specify the $x$- and $y$-directions separately, we denote
\begin{equation}
    \label{eq:F-all-entries}
    F(\z) = \begin{pmatrix} f(z_1, \z) \\ \vdots \\ f(z_n, \z) \end{pmatrix} =
    \begin{pmatrix}
        F_1^x(\z) \\ F_1^y(\z) \\ \vdots \\ F_n^x(\z) \\ F_n^y(\z)
    \end{pmatrix}.
\end{equation}

We use the representation \eqref{eq:F-all-entries} to compute the Jacobian $DF(\z)$ at an arbitrary point $\z=((x_1,y_1),\dotsc,(x_n,y_n))\in\R^{2n}$. It is of the form
\begin{equation*}
    \label{eq:jacobian}
    DF(\z) = (D_{i,j})_{i,j=1}^n \quad \text{with} \quad D_{i,j} = \begin{pmatrix}
        \partial_{x_j}F_i^x(\z) & \partial_{y_j}F_i^x(\z) \\
        \partial_{x_j}F_i^y(\z) & \partial_{y_j}F_i^y(\z)
    \end{pmatrix}.
\end{equation*}
The $2n$ Gershgorin circles of the Jacobian $DF(\z)$ are centered at $\partial_{x_i}F_i^x(\z)$ and $\partial_{y_i}F_i^y(\z)$ for $i=1,\dotsc,n$. Their radii are given by
\[ R_i^x(\z) = \sum_{j\ne i}|\partial_{x_j}F_i^x(\z)| + \sum_{j=1}^n|\partial_{y_j}F_i^x(\z)| \quad \text{and} \quad R_i^x(\z) = \sum_{j\ne i}|\partial_{y_j}F_i^y(\z)| + \sum_{j=1}^n|\partial_{x_j}F_i^y(\z)| \]
respectively.

We compute the partial derivatives as
\begin{equation*}
    \label{eq:partial-derivatives}
    \begin{split}
        \partial_{x_j}F_i^x(\z) &= 
        \begin{cases}
            2 \frac{\varepsilon}{n}b'(\|z_i-z_j\|^2)(x_i-x_j)^2 + \frac{\varepsilon}{n}b(\|z_i-z_j\|^2), \quad & j \ne i, \\[5pt]
            1 - 2 \frac{\varepsilon}{n}\sum_{j\ne i}(b'(\|z_i-z_j\|^2)(x_i-x_j)^2 + b((\|z_i-z_j\|^2)), \quad & j = i,
        \end{cases} \\[5pt]
        \partial_{y_j}F_i^x(\z) &= 
        \begin{cases}
            2 \frac{\varepsilon}{n} b'((\|z_i-z_j\|^2))(x_i-x_j)(y_i-y_j), \quad & j \ne i, \\[5pt]
            2 \frac{\varepsilon}{n} \sum_{j\ne i}b'(\|z_i-z_j\|^2)(x_i-x_j)(y_i-y_j), \quad & j = i,
        \end{cases} \\[5pt]
        \partial_{x_j}F_i^y(\z) &= 
        \begin{cases}
            2 \frac{\varepsilon}{n} b'((\|z_i-z_j\|^2))(x_i-x_j)(y_i-y_j), \quad & j \ne i, \\[5pt]
            2 \frac{\varepsilon}{n} \sum_{j\ne i}b'(\|z_i-z_j\|^2)(x_i-x_j)(y_i-y_j), \quad & j = i,
        \end{cases} \\[5pt]
        \partial_{y_j}F_i^y(\z) &= 
        \begin{cases}
            2 \frac{\varepsilon}{n}b'(\|z_i-z_j\|^2)(y_i-y_j)^2 + \frac{\varepsilon}{n}b(\|z_i-z_j\|^2), \quad & j \ne i, \\[5pt]
            1 - 2 \frac{\varepsilon}{n}\sum_{j\ne i}(b'(\|z_i-z_j\|^2)(y_i-y_j)^2 + b((\|z_i-z_j\|^2)), \quad & j = i.
        \end{cases}
    \end{split}
\end{equation*}
Omitting any details, using these expressions it can be shown that for
\begin{equation}
    \label{eq:eps-estimate}
    \varepsilon < \frac{n}{27(n-1)}
\end{equation}
one has
\begin{equation}
    \label{eq:radii-estimate}
    R_i^x(\z) < | \partial_{x_i}F_i^x(\z) | \quad \text{and} \quad R_i^x(\z) < | \partial_{x_i}F_i^x(\z) |.
\end{equation}
This estimate essentially only requires the triangle inequality and the specific form of the bump function $b$ (see \eqref{eq:ex_bump0}). In particular, \eqref{eq:radii-estimate} shows that none of the Gershgorin circles contains $0$. Therefore, $0$ cannot be an eigenvalue and the Jacobian $DF(\z)$ is invertible for any $\z\in\R^{2n}$. By the inverse function theorem, $F$ is therefore locally invertible at every $\z\in\R^{2n}$.
\begin{remark}
    The estimate \eqref{eq:eps-estimate} is not sharp but sufficient to guarantee the estimate of the radii \eqref{eq:radii-estimate}.
\end{remark}
\section{Additional Proofs for Protocol 2}
\label{ap:contracing-waves}

\label{ap:contracting-wave-proofs}
\lemEgtmInvertable*
\begin{proof}
    \egtmlong\ can be described as $\mathbf{z}^{t+1} = F(\textbf{z}^t)$ with
    \begin{align*}
        F = \left(
        \begin{matrix}
            1 - \epsilon & \frac{\epsilon}{2} & 0 & 0 & 0 & \cdots & 0 & \frac{\epsilon}{2} \\
            \frac{\epsilon}{2} & 1 - \epsilon & \frac{\epsilon}{2} & 0 & 0 & \cdots & 0 & 0 \\
            0 & \frac{\epsilon}{2} & 1 - \epsilon & \frac{\epsilon}{2} & 0 & \cdots & 0 & 0 \\
            0 & 0 & \frac{\epsilon}{2} & 1 - \epsilon & \frac{\epsilon}{2} & \cdots & 0 & 0 \\
            \cdots \\
            \frac{\epsilon}{2} & 0 & 0 & 0 & 0 & \cdots & \frac{\epsilon}{2} & 1 - \epsilon \\
        \end{matrix}
        \right).
    \end{align*}
    For $\epsilon < 0.5$, $F$ is strictly diagonally dominant.
    By \cite{Gersgorin.1931}, $F$ is invertable.
    
\end{proof}

\lemConvexEgtmStaysConvex*
\begin{proof}
    Let us assume $\vec{b}^t$ is a convex set.
    We consider three neighboring robots $\boundary{t}{k-1}, \boundary{t}{k}, \boundary{t}{k+1}$ in $\vec{b}^t$ that are not collinear.
    They form a triangle.
    Let $\tau$ be the target point of $\boundary{t}{k}$ executing \egtm\ with $\epsilon = 1$. The point
    $\tau$ is on the line $l$ between $\boundary{t}{k-1}$ and $\boundary{t}{k+1}$ (black in \Cref{fig:convex-egtm-stays-convex}).
    \begin{figure}[H]
        \includegraphics[page=11]{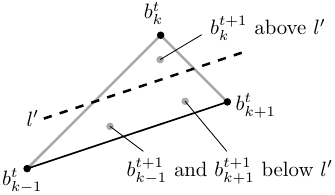}
        \caption{Shows that a convex corner in \wave{t} is convex in \wave{t+1} as well (\Cref{lem:convex-egtm-stays-convex}).}
        \label{fig:convex-egtm-stays-convex}
    \end{figure}
    The position of robot $\boundary{t+1}{k}$ is the target point of \egtm\ with $\epsilon < 0.5$.
    It can be constructed by moving $\tau$ factor $\epsilon$ towards $\boundary{t}{k}$ which is more than half the way.
    Therefore, it must lie above $l'$, the parallel line to $l$ move half way towards $\boundary{t}{k}$ (dotted in \Cref{fig:convex-egtm-stays-convex}).
    
    Let $\tau'$ be the target point of $\boundary{t}{k+1}$ executing \egtm\ with $\epsilon = 1$. 
    Because $\vec{b}^t$ is a convex set, $\boundary{t}{k+2}$ must lie below $l$.
    The midpoint between $\boundary{t}{k}$ and $\boundary{t}{k+2}$ cannot lie above $l'$.
    Moving $\tau'$ towards $\boundary{t}{k+1}$ cannot move it above $l'$.
    Therefore, $\boundary{t+1}{k+1}$ and (analogous) $\boundary{t+1}{k-1}$ lies below $l'$.
    Therefore, the robots $\boundary{t+1}{k-1}, \boundary{t+1}{k}$ and $\boundary{t+1}{k+1}$ form a convex corner in $\vec{b}^{t+1}$.

    If $\boundary{t}{k-1}, \boundary{t}{k}, \boundary{t}{k+1}$ are collinear, they might move onto target points on the same line.
    But with analog arguments as above, it is easy to see that they cannot move onto positions that form a concave corner.

    Therefore, the set $\vec{b}^{t+1}$ is still convex.

    In the proof above we have shown, that $\boundary{t+1}{k}$ lies inside the triangle $\boundary{t}{k-1}, \boundary{t}{k}, \boundary{t}{k+1}$.
    Therefore, the area surrounded by $\vec{b}^{t+1}$ is a subset of the area surrounded by $\vec{b}^t$.
    
\end{proof}

\corWaveSegmentsNotOverlappingNotTwisted*
\begin{proof}
    From the arguments of \Cref{lem:convex-egtm-stays-convex} follows, that each position $\boundary{t+1}{k}$ lies in the triangle $\boundary{t}{k}, \frac{\boundary{t}{k} + \boundary{t}{k-1}}{2}, \frac{\boundary{t}{k} + \boundary{t}{k+1}}{2}$.
    \Cref{fig:not-twisted} shows these triangles with gray dashed lines.
    It is easy to see, that quadrilaterals with corners in these triangles cannot be overlapping or twisted.

    \begin{figure}[H]
        \includegraphics[page=12]{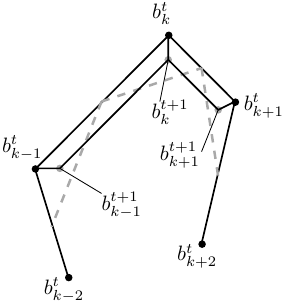}
        \caption{Shows that a convex corner in \wave{t} is convex in \wave{t+1} as well (\Cref{lem:convex-egtm-stays-convex}).}
        \label{fig:not-twisted}
    \end{figure}
\end{proof}

\corNonDegeneratedStaysNonDegenerated*
\begin{proof}
    (1) can be followed from \Cref{lem:convex-egtm-stays-convex}.

    (2) If $\boundary{t}{k}, \boundary{t}{k+1}$ and $\boundary{t}{k+2}$ are collinear but $\boundary{t}{k-1}$ is not, $\boundary{t+1}{k+1} = \boundary{t}{k+1}$ will not move but $\boundary{t+1}{k} \neq \boundary{t}{k}$.
    Therefore, $\boundary{t+1}{k}, \boundary{t+1}{k+1}$ and $\boundary{t+1}{k+2}$ are not anymore collinear and $\waveseg{t+1}{k}$ is non-degenerated.
\end{proof}

\corDistanceToBoundaryInSegment*
\begin{proof}
    We compute the distances between the corners of \waveseg{t}{k} and \waveseg{t+1}{k}.\\

    In the equations below we estimate $|\boundary{t}{k} - \boundary{t}{k+1}| \leq 1$, $|\boundary{t}{k} - \boundary{t}{k+2}| \leq 2$ and $\left|\boundary{t}{k} - \boundary{t}{k+3}\right| \leq 3$.
    \begin{align*}
        \left|\boundary{t}{k} - \boundary{t+1}{k+1}\right|
            &= \left| \boundary{t}{k} - \frac{\epsilon}{2} \boundary{t}{k} - \frac{\epsilon}{2} \boundary{t}{k+2} - (1-\epsilon) \boundary{t}{k+1} \right|
            \\
            &= \frac{\epsilon}{2}\left|\boundary{t}{k} - \boundary{t}{k+2}\right| + (1-\epsilon) \left|\boundary{t}{k} + \boundary{t}{k+1}\right|
            \\
            &\leq \epsilon + (1-\epsilon) = 1
    \end{align*}

    Analog is $\left|\boundary{t}{k+1} - \boundary{t+1}{k}\right| \leq 1$.
    It is clear, that $\left|\boundary{t}{k} - \boundary{t+1}{k}\right| \leq \epsilon$ and $\left|\boundary{t}{k+1} - \boundary{t+1}{k+1}\right| \leq \epsilon$.\\
    
    \begin{align*}
        \left|\boundary{t}{k} - \boundary{t+2}{k+1}\right|
            &= \left|\boundary{t}{k} - (1-\epsilon) \boundary{t+1}{k+1} - \frac{\epsilon}{2} \boundary{t+1}{k} - \frac{\epsilon}{2}\boundary{t+1}{k+2}\right|
            \\
            &\leq (1-\epsilon) \cdot 1 + \frac{\epsilon}{2} \cdot \epsilon + \frac{\epsilon}{2} \left|\boundary{t}{k} - \frac{\epsilon}{2} \boundary{t}{k+1} - \frac{\epsilon}{2} \boundary{t}{k+3} - (1-\epsilon) \boundary{t}{k+2}\right|
            \\
            &\leq (1 - \epsilon) + \frac{\epsilon^2}{2} + \frac{\epsilon}{2} (\frac{\epsilon}{2} + 3\frac{\epsilon}{2} + 2(1-\epsilon))
            \\
            &= (1+\epsilon)(1-\epsilon) + \frac{3}{2} \epsilon^2
            \\
            &= 1 + \frac{\epsilon^2}{2}
    \end{align*}

    Analog is $\left|\boundary{t}{k+1} - \boundary{t+2}{k}\right| \leq 1 + \frac{\epsilon^2}{2}$.
    It is clear, that $\left|\boundary{t}{k} - \boundary{t+2}{k}\right| \leq 2\epsilon$ and $\left|\boundary{t}{k+1} - \boundary{t+2}{k+1}\right| \leq 2\epsilon$.
    
    Any robot inside \waveseg{t}{k} or \waveseg{t+1}{k} as a smaller or equal distance to \boundary{t}{k} and \boundary{t}{k+1} than the farthest corner of the wave-segments.
    Therefore, the distance is $\leq 1 + \frac{\epsilon^2}{2}$.
\end{proof}

\end{document}